\documentclass[letterpaper, 10 pt, conference]{ieeeconf}  

\IEEEoverridecommandlockouts                              

\usepackage{graphics} 
\usepackage{epsfig} 
\usepackage{mathptmx} 
\usepackage{times} 
\usepackage{amsmath} 
\usepackage{amssymb}  
 
\usepackage{algorithm}
\usepackage{algorithmic}
\usepackage{amsthm}

\newtheorem{theorem}{Theorem}
\newtheorem{proposition}{Proposition}
\usepackage{graphicx}
\usepackage{subfigure}

\title{\LARGE \bf
Harmonic Field-based Provable Exploration of 3D Indoor Environments
}

\author{Raksi Kopo, Charalampos P. Bechlioulis, Kostas J. Kyriakopoulos
\thanks{Raksi Kopo and Kostas J. Kyriakopoulos are with School of Mechanical Engineering, Control Systems Laboratory, National Technical University of Athens, 11144 Athens, Greece (e-mail: mc17067@mail.ntua.gr;kkyria@mail.ntua.gr).
}
\thanks{Charalampos P. Bechlioulis is with the Department of Electrical and Computer Engineering, University of Patras, 26504 Patra, Greece (e-mail: chmpechl@upatras.gr).
}
}

\begin{document}
\maketitle

\thispagestyle{empty}
\pagestyle{empty}

\begin{abstract}
This work presents a safe and efficient methodology for autonomous indoor exploration with aerial robots using Harmonic Potential Fields (HPF). The challenge of applying HPF in complex 3D environments rests on high computational load involved in solving the Laplace equation. To address this issue, the proposed solution utilizes the Fast Multiple accelerated Boundary Element Method with boundary values controlled to ensure both safety and convergence. The methodology is validated through simulations that demonstrate its efficiency, safety and convergence.

\end{abstract}


\section{Introduction}
Autonomous robotic exploration has a range of useful applications, including search and rescue missions, infrastructure inspection, and mapping. Robots operating autonomously can reduce the risks and costs associated with human exploration in hazardous and remote environments. Several methods have been developed aiming to address the autonomous exploration problem. A popular approach is maximizing the expected information gain \cite{Bourgault, Stachniss, Bai}. In general, this involves finding the control action that minimizes a utility function, which encompasses the map's uncertainty, the uncertainty in the robot's position, and possibly a type of cost for executing the control action, usually the path length. In \cite{Bircher} Rapidly Random exploring Trees (RRT) path planning is introduced in information gain exploration. The authors in \cite{Selin} use Frontier-based Exploration planning \cite{Yamauchi} for global exploration and the exploration method in \cite{Bircher} to improve the performance of the latter in larger environments.

An alternative approach utilizes Artificial Potential Fields (APF) in exploration. Compared to RRTs, which ensure probabilistic exploration, meaning full exploration given infinite time, APFs can ensure complete exploration if they are free of local minima. Furthermore, APFs can produce smooth trajectories and are inherently reactive. The two main approaches in constructing free of local minima APF are Navigation Functions \cite{Rimon} and Harmonic Potential Fields \cite{Connoly}. The former requires parameter tuning, while the latter is inherently free of local minima inside their domain of definition. The main drawbacks of such approaches are that they might require extraction of a closed form workspace representation from the probabilistic map or, in the case of HAPF, they can be computationally chalenging to calculate using conventional methods like Finite Elements or Finite Differences. These problems become more difficult in three dimensions due to more complex geometries and larger workspace volume. The method in \cite{Grontas} uses Fast Multiple Boundary Element Method (FastBEM) which significantly accelerates the calculation of the harmonic potential. In another work \cite{Shade}, the author uses GPU-parallelised Finite Differences and devises an algorithm to eliminate the explored flat potential areas from future calculations.

In this work, we extend the results presented in \cite{Grontas} to three-dimensional complex indoor environments. The proposed exploration scheme guarantees complete exploration of such environments in finite time, while producing smooth and safe trajectories. We use Fast Multiple Boundary Element Method (FastBEM) to efficiently calculate the Harmonic Potential Field navigating the robot. Contrary to [11], we have developed algorithms for the removal of explored dead ends and robust boundary extraction, which further improve the exploration process. Finally, we demonstrate the effectiveness of our approach by testing it in a realistic indoor environment.

\section{Problem Formulation}
The initially unknown workspace $W$ is a compact and connected subspace of $\Re^3$. The interior of $W$ is the free space and its boundary can be considered as a union of disconnected closed surfaces with multiple handles. In most cases, there is only one surface if there are no floating or flying obstacles present. We assume that the robot has spherical sensing capabilities, meaning that its sensing space $S(p)$ is composed of all line segments starting at the robot's position $p=[x,y,z]$ with a length no greater than a specified range $r$ and lying inside $W$. At a given time instance $t$, the explored space $\varepsilon(t)$ is defined as the union of sensing spaces $S(p)$ for all positions $p$ along the path $P(t)$ that the robot has traversed up to that time. More specifically, the explored space at a time instance $t$ is defined as $\varepsilon(t)=\bigcup_{p\in P(t)} S(p)$, where $P(t)$ is the path that the robot has traversed in time $t$.  The boundary of the explored space consists of two different sets: the occupied boundary $\partial\varepsilon_o$, which is comprised of parts of $\partial W$ and the free boundary parts $\partial \varepsilon_f$, which belong to the unoccupied space. The former remains the same while the latter is integrated into the interior of $\varepsilon$ as the robot explores the workspace.
\\
\indent \underline{\emph{Problem definition}}: Assuming a robot that obeys a single integrator kinematic model $\dot{p}=u$ with initial position $p(0)=p_0 \in intW$, find a control law $u=f(p,t,\varepsilon)$ that generates a path $P(T)$, so that in finite time $T>0$ , $\varepsilon(T)=W$.

\section{Methodology}
The central idea behind boundary value problem (BVP) exploration is to find a harmonic potential field $\varphi(p)$ by solving the Laplace equation \eqref{LaplaceEquation}, subject to boundary conditions that assign attractive potential unexplored frontiers and repulsive potential to boundaries corresponding to obstacles:
\begin{subequations}
\begin{equation}
    \nabla^2\varphi(p)=0,  
    \label{LaplaceEquation}
\end{equation}
  \begin{equation}
    \frac{\partial\varphi}{\partial n}={\nabla\varphi}\cdot n(q) =k(q),~\forall q\in\partial\varepsilon
    \label{NeumannBoundaryConditions}
\end{equation}
\label{BVP}
\end{subequations}
\noindent where $n(q)$ denotes the unit normal vector of $\partial\varepsilon$ at $q$ and $k(q)$ is a given function that evaluates the directional derivative of $\varphi(p)$ over the boundary $\partial \varepsilon$. 

 For the boundary value problem \eqref{BVP} with Neumann conditions \eqref{NeumannBoundaryConditions} to have a solution, the potential function must satisfy the compatibility constraint. This constraint states that the flux (flow) of the gradient field of $\varphi$ across the boundary of the region must be equal to zero. 
\begin{equation}
\iint \limits_{\partial\varepsilon}\nabla\varphi \cdot n\ dS=0
    \label{CompatibilityConstraint}
\end{equation}
An important property of harmonic functions, namely the maximum principle, states that a non-constant harmonic function cannot attain a maximum or minimum at an interior point of its domain. This implies that the values of a harmonic function in a bounded domain are bounded by its maximum and minimum values on the boundary. This makes harmonic potential fields suitable for autonomous navigation since the robot cannot be trapped in a local minimum.

\subsection{Fast Multi-pole Boundary Element Method}
The method chosen to solve the Laplace equation is the Fast Multi-pole accelerated Boundary Element Method (FastBEM) developed in \cite{liu_2009} which accelerates the conventional BEM and it is more efficient than Finite Element Method and Finite Difference Method in terms of computational time and memory requirements and for complex boundary geometries such as those of indoor environments. In conventional BEM \cite{Katsikadelis2016} the boundary of the domain $\partial\varepsilon$ is discretized and the boundary value problem \eqref{BVP} is reformulated to boundary integral equations \eqref{BoundaryIntegralEquations} using Green's second identity and the fundamental solution. More specifically, the solution at a point $p\in\varepsilon$, inside the domain is expressed as a linear combination of the boundary values of the potential and its normal derivative:
\begin{subequations}
\begin{equation}
\varphi (p)= H_p\cdot\hat{\varphi}-G_p\cdot\hat{k}
    \label{PotetialInsideDomain}
\end{equation}    
For Neumann BVPs, the unknown boundary values of the potential, $\hat{\varphi}=[\hat{\varphi}_1, \hat{\varphi}_2, ..., \hat{\varphi}_N]^T$, depend linearly on the boundary values of its normal derivative, $\hat{k}=[\hat{k}_1, \hat{k}_2, ...,\hat{k}_N]^T$
\begin{equation}
  \hat{\varphi}= H^{-1}\cdot G \cdot \hat{k}  
  \label{PotentialOnBoundary}
\end{equation}
where $N$ is the number of boundary elements. The definitions of the $1\times N$ vectors $H_p$, $G_p$ and the $N\times N$ matrices $H$, $G$ which contain the surface integrals of the fundamental solution and its normal derivative on each element can be found in \cite{Katsikadelis2016}. The computational load of the Boundary Element Method primarily lies in assembling the BEM matrix $H$ and solving the associated linear system. It is worth noting that the entries of $H$ depend only on the geometry of the boundary, so if certain parts of the boundary do not change, the corresponding entries of the matrix are computed only once and reused, which reduces the computational effort significantly.

 In conventional BEM the computation of the matrix $H$ requires $O(N^2)$ operations and the solution of the system $H\cdot\hat{\varphi}=G\cdot\hat{k}$ using direct solvers requires $O(N^3)$ operations. The main idea of FastBEM is to use iterative solvers with the Fast Multi-pole Method to accelerate the matrix-vector multiplication $H\cdot\hat{\varphi}$ in each iteration without ever forming the entire matrix $H$ explicitly. The boundary elements are subdivided via an octree decomposition and the interactions among elements in conventional BEM are replaced by cell to cell interactions of the octree. Direct integration is performed only for the elements that are close to the source point whereas multi-pole  expansions are used for far away elements. This reduces the number of operations to $O(N)$ decreasing dramatically the solution time. Moreover, the use of iterative solvers reduces significantly the memory requirements since the elements of $H$ do not need to be explicitly stored.
\label{BoundaryIntegralEquations}
\end{subequations}
\subsection{Occupancy Grid Mapping}
The representation of the explored space $\varepsilon$ is done through occupancy grid maps \cite{ProbabilisticRobotics}. An occupancy probability value $Pr(m)$ is assigned to each grid cell and updated using algorithms similar to those described in Chapter 9 of \cite{ProbabilisticRobotics}, extended to three dimensions. Each point $q\in\partial\varepsilon$ inherits the probability of the cell that contains it, $Pr(q)=Pr(m(q))$, $\parallel c-q\parallel_\infty\leq m_r/2$, where $c$ denotes the center of cell $m$ and $m_r$ the cell size of the map in meters per cell. Defining an occupancy probability threshold $\Bar{a}\in[0.5,1]$ we can classify each boundary point $q\in\partial\varepsilon$ as occupied or free and define the sets $\partial\varepsilon_o$ and $\partial\varepsilon_f$ as $\partial\varepsilon_o=\{q\in\partial\varepsilon: Pr(q)\geq \Bar{a}\}$ and $\partial\varepsilon_f=\partial\varepsilon\setminus\partial\varepsilon_o$.
\subsection{Velocity Control Law}
The robot's velocity control law has the form:
\begin{equation}
    u=-K_u\cdot s\cdot \nabla_p\varphi(p,\hat{k})
    \label{VelocityControlLaw}
\end{equation}
where $K_u$ is a scaling constant and $s=S_{R_1}(d(p,\partial\varepsilon))$ denotes a function that asymptotically approaches zero as the robot approaches the boundary, for instance: 
\begin{equation}
 \ S_a(x)=\left\{
  \begin{array}{ll}
  1, & x>a \\
  3(\frac{x}{a})^2-2(\frac{x}{a})^3, & 0\leq x\leq a\\
  0, & x<0.\\
  \end{array}
  \right.
   \label{BumpFunction} 
\end{equation}
By the definition of $S_{R_1}$ we see that whenever the minimum euclidean distance $d(p,\partial\varepsilon)$ between the robot and the the boundary becomes less that a constant $R_1$, the robot will start slowing down.
\subsection{Construction of the Boundary Value Problem}
In order to guarantee safety and complete exploration via \eqref{VelocityControlLaw}, as it will be shown later in the next section, the boundary values $\hat{k}$ will track their desired target values $\hat{k}_t$ that will be designed in the sequel.
\subsubsection{Target Boundary Values}
First, we will define $\hat{k}_t$ as:
\begin{equation}
  {\hat{k}_{t}=\frac{\Bar{k}}{k'_m I(q)}k'(q)}.
  \label{TargetBoundaryValues}
\end{equation}
\begin{subequations}
The function $k'(q)$ is defined using \eqref{BumpFunction} as:
\begin{equation}
  k'=\left\{
  \begin{array}{ll}
  S_{1-\alpha}(Pr(q)-\alpha), & Pr(q)\geq\alpha \\
  -S_{\alpha}(\alpha-Pr(q)), & Pr(q)<\alpha\\
  \end{array}
\right.
    \label{k1}
\end{equation}
and it assigns the appropriate sign to each target boundary value based on the probability occupancy $Pr(q)$ of the corresponding boundary element at position $q$ and the occupancy threshold $\alpha\in [0.5, 1]$. The function $I(q)$ defined as
\begin{equation}
  I(q)=\left\{
  \begin{array}{ll}
  \iint\limits_{\partial\varepsilon_p}{k'}{dS}, & q\in\partial\varepsilon_p\\
  \\
  -\iint\limits_{\partial\varepsilon_n}{k'}{dS}, & q\in\partial\varepsilon_n\\
  \end{array}
  \right.
  \label{I}
\end{equation}
is used to render the target boundary values compatible as shown in (\ref{CompatibilityConstraint}) for the feasibility of the BVP solution. The sets $\partial\varepsilon_p=\{q\in\partial\varepsilon : k'(q)\geq0\}$ and $ 
\partial\varepsilon_n=\{q\in\partial\varepsilon : k'(q)<0\}$ correspond to the parts of the boundary where $k'$ is positive and negative respectively.
Regarding $\Bar{k}$ and $k'_m$, it should be noted that $\Bar{k}$ is a constant that specifies the maximum and minimum values of $\hat{k}_t$ and $k'_m=max_{q\in\partial\varepsilon}(\parallel \frac{k'(q)} {I(q)}\parallel)$.
    \label{TargetBoundaryValuesAuxiliaryFunctions}
\end{subequations}
\subsubsection{Boundary Value Control Law}
The boundary values $\hat{k}$ converge to the desired target values $\hat{k_t}$ through the following control law:
\begin{equation}
  {\dot{\hat{k}}=(c\mu+b_e)(\hat{k}_t-\hat{k}) } ,\  {\hat{k}(0)=\hat{k}_0}
  \label{BoundaryValueControlLaw}
\end{equation}
The parameters $c$ and $\mu$ are defined as $c=S_{\epsilon_w}(\parallel\nabla_p\varphi-\epsilon_1\parallel)$ and $\mu=S_{\mu_1}\left(\frac{K_u s\parallel\nabla_p\varphi\parallel^2}{\parallel\frac{\partial\varphi}{\partial \hat{k}}(\hat{k}_t-\hat{k})\parallel+\epsilon_2}\right)$ with $\epsilon_w$, $\epsilon_1$ and $\epsilon_2$ small positive constants. The term $c$ ensures that $\dot{\hat{k}}\rightarrow 0$ as the robot converges to a critical point of $\varphi$, whereas $\mu$ is used to adjust the rate with which the boundary values approach the target values. By specifying $\mu_1 > 9/8$ we ensure that $S_{\mu_1}(x)<x$, $\forall x > 0$. The usefulness of this will become apparent in the next section. Regarding the term $\frac{\partial\varphi}{\partial \hat{k}}=[\frac{\partial\varphi}{\partial \hat{k}_1}, \frac{\partial\varphi}{\partial \hat{k}_2}..., \frac{\partial\varphi}{\partial \hat{k}_N}]$ in $\mu$, we see that combining \eqref{PotetialInsideDomain} and \eqref{PotentialOnBoundary} we have $\varphi = (H_p\cdot H^{-1}\cdot G-G_p)\cdot\hat{k}$, therefore $\frac{\partial\varphi}{\partial\hat{k}}=H_p\cdot H^{-1}\cdot G-G_p$. Finally, $b_e$ is defined as
\begin{equation}
b_e=\left\{
  \begin{array}{ll}
  0, & if \ \hat{k}(q) \geq0 \  \forall q\in\partial\varepsilon_o\\
  \\
  1, & \text{otherwise}\\
  \end{array}
  \right.  
    \label{bumpParam}
\end{equation}
and ensures that the potential becomes safe by increasing the convergence rate of $\hat{k}$ to the desired target value $\hat{k}_t$ in case there exists an attractive occupied boundary, such as when a previously free part of the boundary is replaced by an obstacle during exploration.
\subsection{Proof of Safety and Complete Exploration}
In this subsection, we state some useful theoretical results regarding the aforementioned control design.
\begin{proposition}
 $\hat{k}_t$ is bounded and compatible.
\end{proposition}

\begin{proof} From the definition of $k'_m$ it holds that
    $$k'_m\geq\parallel\frac{k'(q)}{I(q)}\parallel
    \Rightarrow \parallel\frac{\Bar{k}k'(q)}{k'_m I(q)}\parallel\leq\Bar{k}
    \Rightarrow -\Bar{k}\leq\hat{k}_t\leq\Bar{k}$$
    therefore $\hat{k}_t$ is bounded. Also
$$\iint\limits_{\partial\varepsilon}\hat{k}_t{dS}=\iint\limits_{\partial\varepsilon}\frac{\Bar{k}k'(q)}{k'_m I(q)}{dS}=\iint\limits_{\partial\varepsilon_p}\frac{\Bar{k}k'(q)}{k'_m I(q)}{dS}+\iint\limits_{\partial\varepsilon_n}\frac{\Bar{k}k'(q)}{k'_m I(q)}{dS}$$
    \\
   $$ =\frac{\Bar{k}}{k'_m {\iint\limits_{\partial\varepsilon_p}{k'}{dS}}}{\iint\limits_{\partial\varepsilon_p}k'(q){dS}}-\frac{\Bar{k}}{k'_m {\iint\limits_{\partial\varepsilon_n}{k'}{dS}}}{\iint\limits_{\partial\varepsilon_n}k'(q){dS}}=\frac{\Bar{k}}{k'_m}-\frac{\Bar{k}}{k'_m}=0$$
 which completes the proof.
\end{proof}
\begin{proposition}
 Assuming that $\hat{k}$ is compatible for t = 0, the adaptive law \eqref{BoundaryValueControlLaw} guarantees that $\hat{k}$ will remain compatible and bounded for all time.   
\end{proposition}

\begin{proof}
    Assume that $\hat{k}_i(t)=\bar{k}$ for some time \ $t \geq 0$. Then, $\dot{\hat{k}}_i$ must be non-positive because it is proportional to $\hat{k}_{t,i}-\hat{k}_i$ and $\hat{k}_{t,i}\leq\Bar{k}$ thus,\ $\hat{k}_i$ cannot exceed $\Bar{k}$. Similarly, it holds that $-\Bar{k}\leq\hat{k}_i$. Therefore $\hat{k}$ is bounded. Next, we denote as $I_c$ the closed surface integral of $\hat{k}$. For constant boundary elements
   $$I_c=\iint\limits_{\partial\varepsilon}\hat{k}{dS}=\sum\limits_{i=1}^N A_i \hat{k}_i=A\cdot\hat{k}$$ where $A=[A_1, A_2, ..., A_N]$ denotes the areas of each element. In case $A_i=A_j, i\neq j$ then we remove the j-th elements from $A$ and $\hat{k}$ and replace $\hat{k}_i$ with $\hat{k}_i+\hat{k}_j$ so that the new $A$ is full column rank. Taking the derivative of $I_c$ w.r.t. time we have
\[\dot{I_c}=\dot{A}\cdot\hat{k}+A\cdot\dot{\hat{k}}=\dot{A}\cdot\hat{k}+A\cdot (c\mu+b_e)(\hat{k}_{t}-\hat{k})\] 
    and from the compatibility of $\hat{k}_t$, i.e., $A\cdot\hat{k}_{t}=0$, we obtain:
    \[\dot{I_c}=\dot{A}\cdot\hat{k}-A\cdot\hat{k}(c\mu+b_e)=[\dot{A}\cdot(A^T\cdot A)^{-1}\cdot A^T-(c\mu+b_e)] A\cdot\hat{k} \] 
    \[
    \dot{I_c}=[\dot{A}\cdot(A^T\cdot A)^{-1}\cdot A^T-(c\mu+b_e)]I_c.\]
    Therefore if $k$ satisfies the compatibility condition, $I_c=0$, at $t=0$ then it satisfies it for all $t\geq0$. Considering this, we choose $\hat{k}_0=\hat{k}_{t,0}$ as initial conditions for \eqref{BoundaryValueControlLaw}. 
    \label{proof:Proposition 2}
\end{proof}
\begin{proposition}
    If the potential $\varphi(p,\hat{k})$ is unsafe, the control law \eqref{BoundaryValueControlLaw} guarantees it will become safe in finite time.
\end{proposition}
\begin{proof}
When the robot is using the potential field computed in the previous time step of the newly explored occupied boundary, $\partial\varepsilon_o$ may have attractive potential rendering $\varphi$ unsafe. In that case, $b_e$ in \eqref{bumpParam} is activated so that in finite time $\hat{k}$ approaches the desired positive value $\hat{k}_{t}$ asymptotically and the potential field computed in the next step becomes safe eventually.
\label{proof:Proposition 3}
\end{proof}
\begin{proposition}
    The function $\varphi(p,\hat{k})$ is lower bounded.
\end{proposition}
\begin{proof}
    Due to well-posedness of the boundary value problem $\varphi(p,\hat{k})$ depends continuously on the boundary conditions $ \hat{k}\in[-\Bar{k},\Bar{k}], \ \forall p\in\varepsilon$. The set $[-\Bar{k},\Bar{k}]$ is closed therefore by the extreme value theorem $\varphi$ accepts a maximum and a minimum value in $\{(p,\hat{k})| \hat{k}\in[-\Bar{k},\Bar{k}] \ and \ p\in\varepsilon\}$. Therefore it is lower bounded.
    \label{proof:Proposition 4}
\end{proof}
\begin{proposition}
    The robot's trajectory under control law \eqref{VelocityControlLaw} is safe.
\end{proposition}
\begin{proof}
    If $\varphi$ is safe then all points in the occupied boundary are repulsive and collision avoidance is ensured. In case $\varphi$ is not safe, it will become safe in finite time as $p\rightarrow\partial  \varepsilon_o$, which means that the robot approaches the boundary asymptotically, therefore in the worst case the robot stops at the boundary. Assuming accurate sensing and that the updated potential field is computed as the robot moves, the occupied boundary becomes strictly repulsive and collision avoidance is ensured.  
     \label{proof:Proposition 6}
\end{proof}

\begin{theorem}
 Assuming $b_e=0$, the proposed robot control guarantees the full exploration of $W$ in finite time from almost any initial configuration.
\end{theorem}
\begin{proof}
Barbalat's lemma will be used to prove that $\frac{d\varphi}{dt}\rightarrow0$ as $t\rightarrow\infty$. The following conditions must be satisfied: 1) $\varphi$ has a finite limit as $t\rightarrow\infty$, and 2) $\frac{d\varphi}{dt}$ is uniformly continuous w.r.t. time. For the first condition we showed that $\varphi(p,\hat{k})$ is lower bounded and we need to show that $\frac{d\varphi}{dt}\leq0$, where:
\[\frac{d\varphi}{dt}=\nabla_p\varphi\cdot\frac{dp}{dt}+\frac{\partial\varphi}{\partial\hat{k}}\cdot\frac{d\hat{k}}{dt}\]
\begin{equation}
 \frac{d\varphi}{dt}\leq-K_us\parallel\nabla_p\varphi\parallel^2+\parallel\frac{\partial\varphi}{\partial\hat{k}}\cdot(\hat{k}_t-\hat{k})\parallel c\mu.\\
\label{NegDerivative}
\end{equation}
For $\mu_1>9/8$ it holds that:
$$\frac{K_us\parallel\nabla_p\varphi\parallel^2}{\parallel\frac{\partial\varphi}{\partial \hat{k}}(\hat{k}_t-\hat{k})\parallel+\varepsilon_2}\geq S_{\mu_1}\left(\frac{K_us\parallel\nabla_p\varphi\parallel^2}{\parallel\frac{\partial\varphi}{\partial \hat{k}}(\hat{k}_t-\hat{k})\parallel+\varepsilon_2}\right)$$
$$ K_us\parallel\nabla_p\varphi\parallel^2\geq (\parallel\frac{\partial\varphi}{\partial \hat{k}}(\hat{k}_t-\hat{k})\parallel+\varepsilon_2)\mu$$
$$ K_us\parallel\nabla_p\varphi\parallel^2\geq \parallel\frac{\partial\varphi}{\partial \hat{k}}(\hat{k}_t-\hat{k})\parallel c \mu$$
$$ 0\geq-K_us\parallel\nabla_p\varphi\parallel^2+\parallel\frac{\partial\varphi}{\partial \hat{k}}(\hat{k}_t-\hat{k})\parallel c \mu.$$
Next, we show that $\frac{d\varphi}{dt}$ is uniformly continuous w.r.t. time. It suffices to show that $\frac{d\varphi}{dt}$ is locally Lipschitz continuous since $\varphi$ has been proven bounded.
In the first term of $\frac{d\varphi}{dt}$ in \eqref{NegDerivative} , $s$ is locally Lipschitz continuous and bounded between 0 and 1, $\nabla_p\varphi$ is real analytic and thus locally Lipschitz and bounded in $\varepsilon$. Therefore, the first term of $\frac{d\varphi}{dt}$ is locally Lipschitz continuous as a product of locally Lipschitz continuous and bounded functions. Regarding the second term, it can be shown that it is continuous and has a piece-wise continuous first derivative, as a composition of such functions. Using the extreme value theorem, it can be shown that the derivative of the second term is bounded. Thus, using the mean value theorem it can be shown that the second term of $\frac{d\varphi}{dt}$ is locally Lipschitz continuous. Consequently, $\frac{d\varphi}{dt}$ is locally Lipschitz continuous.

Therefore the conditions of Barbalat's lemma are fulfilled and the derivative of $\varphi$ w.r.t. time converges to 0 as $t\rightarrow\infty$. By \eqref{NegDerivative} we see that $\frac{d\varphi}{dt}\rightarrow 0$ occurs in two cases: i) when the robot approaches the free boundary ($s\rightarrow0$ and $\mu\rightarrow0$) and ii) when $\parallel\nabla_p\varphi\parallel\rightarrow0$ ($c\rightarrow0$).
The second case, means that the robot converges to a critical point of $\varphi$ which for a harmonic function is an unstable saddle. Therefore, the robot will move towards the free boundary of $W$ and thus fully explore it in finite time.
\end{proof}
\section{Implementation}
\subsection{Boundary Extraction}
After obtaining the occupancy map of the explored workspace, we need to extract its boundary in form of a closed triangulated surface which is required by the BEM solver. In this implementation, we used 3D Delaunay triangulation \cite{DelaunayMeshGeneration} due to its computational efficiency and ability to produce the required closed surface. The implementation steps are given as follows:

\begin{algorithm} Boundary Extraction

\begin{enumerate}
\item Extract the centers of the cells in the free connected space that contains the robot.
\item Find the 3D Delaunay triangulation (tetrahedrization) of that set of points.
\item Remove the tetrahedra that do not correspond to the free space. (Due to the dense grid structure of the free space cell centers, we can remove the tetrahedra with volume and/or maximum edge length greater than a threshold, which are defined as $m_r^3/2$ and $\sqrt{3}m_r$, respectively.)
\item Extract the boundary facets and vertices of the remaining tetrahedrization, which form the closed surface mesh of the explored space’s boundary.
\end{enumerate}
\end{algorithm}

\subsection{Dead End Closing}
Given a generated map and trajectory, the goal of this algorithm is to remove the explored regions that the robot will not have to traverse again (i.e., areas with no unexplored free boundaries). To do this we need to identify boundary patches between the region to be removed and the remaining space. Once the region has been removed, the boundaries between the remaining space and the removed region are set to have an occupancy probability of 0.99, indicating that they are now considered occupied. For simplicity, the algorithm considers only one boundary patch at a time that separates the explored space into two distinct regions. In this implementation, we examine planar patches that intersect the trajectory and are perpendicular to it. Next, we provide the implementation steps of the algorithm and demonstrate the results from its application on the ''Maze" environment in Figure \ref{fig:DeadEndResults}. The figure shows the boundary of the explored space on the left and the boundary that is used in the Fast BEM solver after applying the dead end closing algorithm on the right.
\begin{figure}
    \centering
    \includegraphics[width=\linewidth]{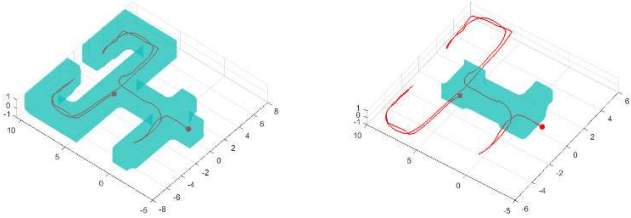}
    \caption{Removal of dead end areas.}
    \label{fig:DeadEndResults}
\end{figure}
\begin{algorithm} Dead-end Closing
\\\\For all trajectory points
\begin{enumerate}
    \item Find the planar patch that contains the trajectory point and is perpendicular to the trajectory.
    \item If it splits the explored space into exactly two spaces proceed to step 3 else proceed to the next trajectory point.
    \item If the space that does not contain the robot and does not have any unexplored free boundary then proceed to step 4 else proceed to the next trajectory point.
    \item Set the occupancy of the region that does not contain the robot and the planar patch to 0.50 and 0.99, respectively.
\end{enumerate}
\end{algorithm}

\section{Results}
In this section, we present the simulation results from our implementation. The sensing and occupancy grid mapping were simulated using Matlab's Navigation Toolbox. The environments were designed in CAD software and imported as mesh files. All aforementioned algorithms where implemented in Matlab 2020b, except the Laplace solver which is provided in \cite{FastBEMSoftware}. The controller parameter values that are used in the following simulations are shown in TABLE \ref{tab:ControllerParameters}. We test our algorithm in two different environments shown in Figure \ref{fig:Environments}. The first is a small corridor maze in order to demonstrate the performance of the dead end closing algorithm and the second is a more complex and larger environment representing a two story building with multiple rooms. 
\begin{figure}
\centering
\subfigure[Corridor Maze]{\includegraphics[width=0.3\linewidth]{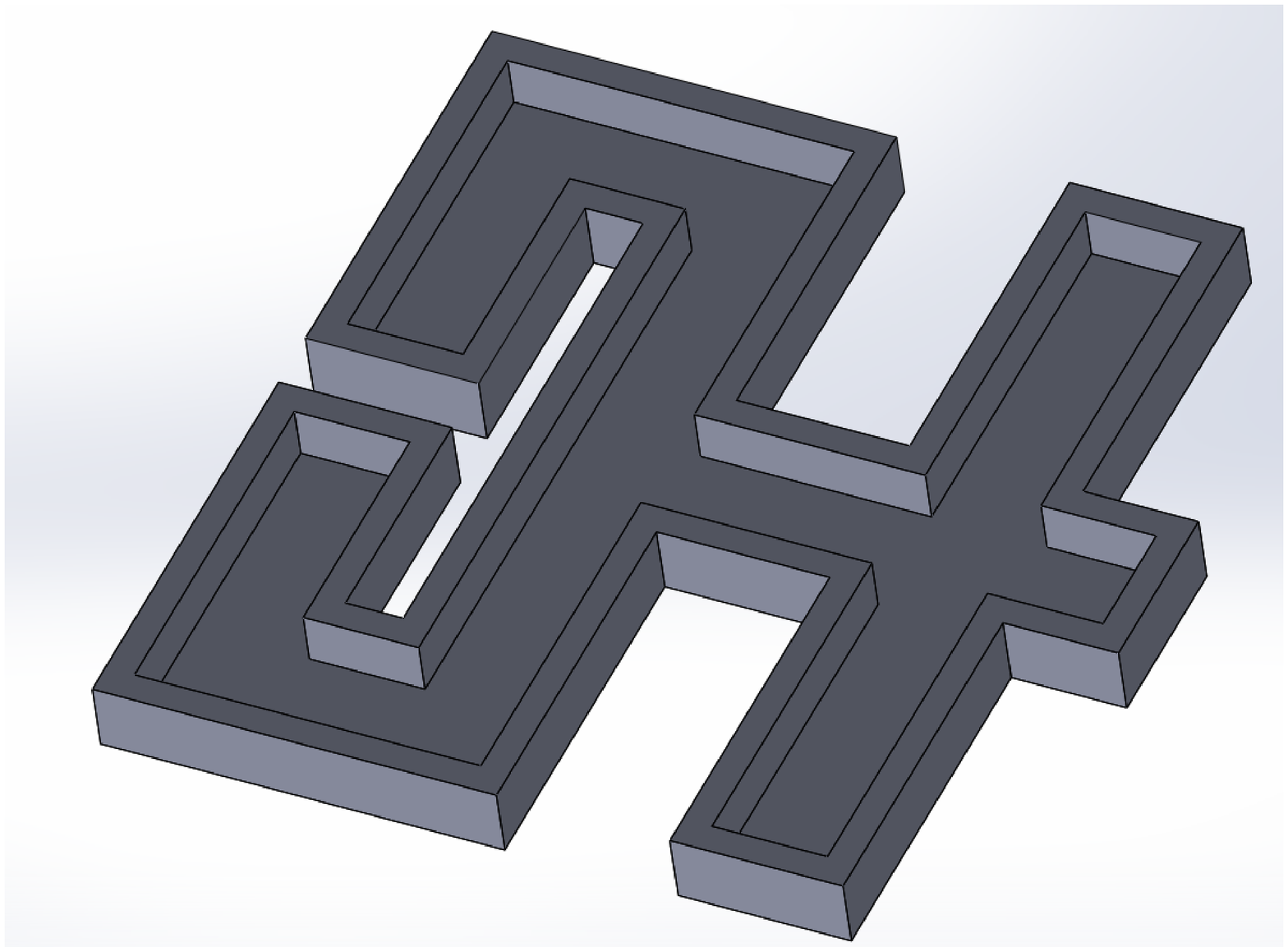}\label{fig:CorMazeEnv}}
\hfill
\subfigure[Two Story Upper Floor]{\includegraphics[width=0.3\linewidth]{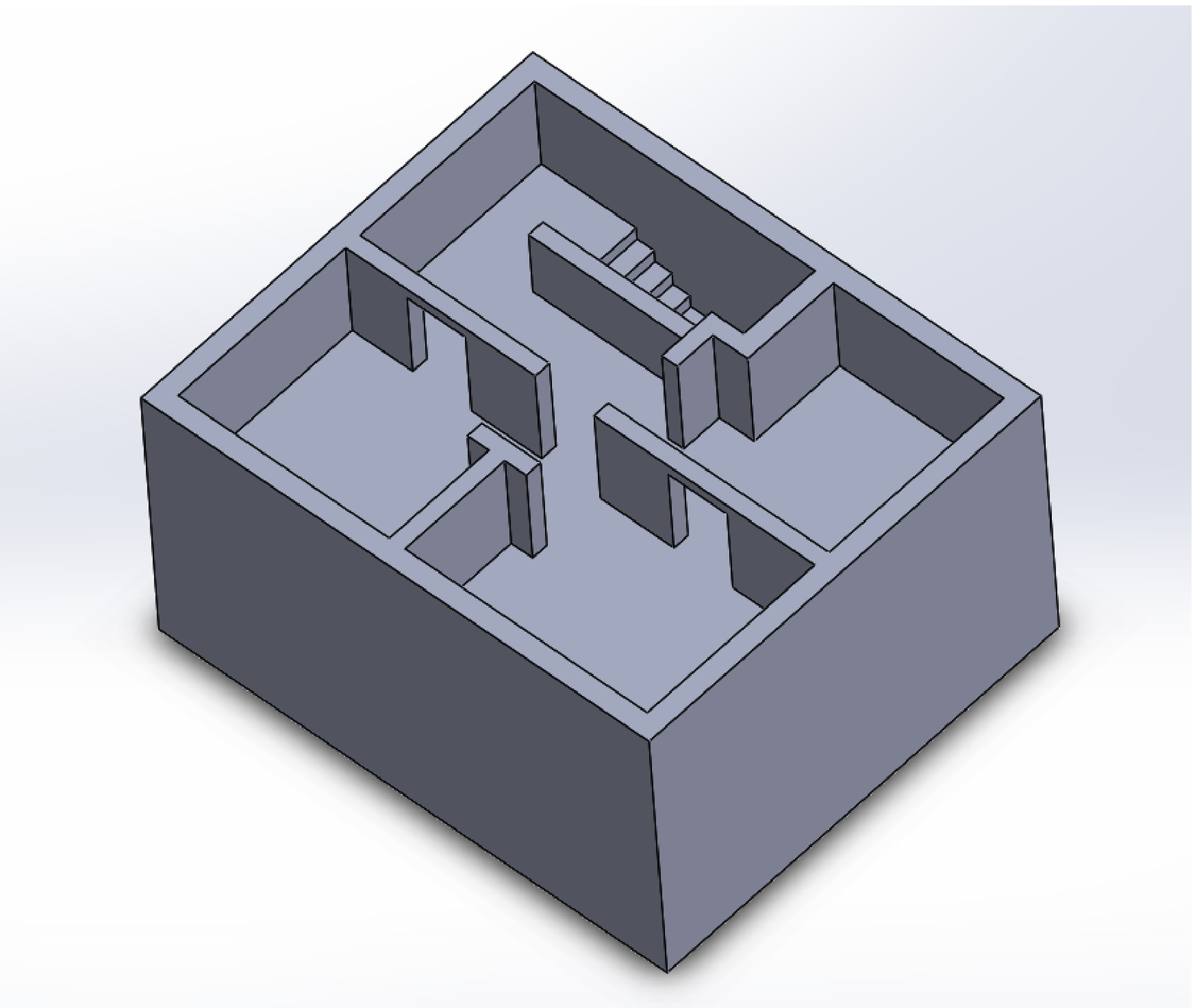}\label{fig:TwoStoryUFEnv}}
\hfill
\subfigure[Two Story Lower Floor]{\includegraphics[width=0.3\linewidth]{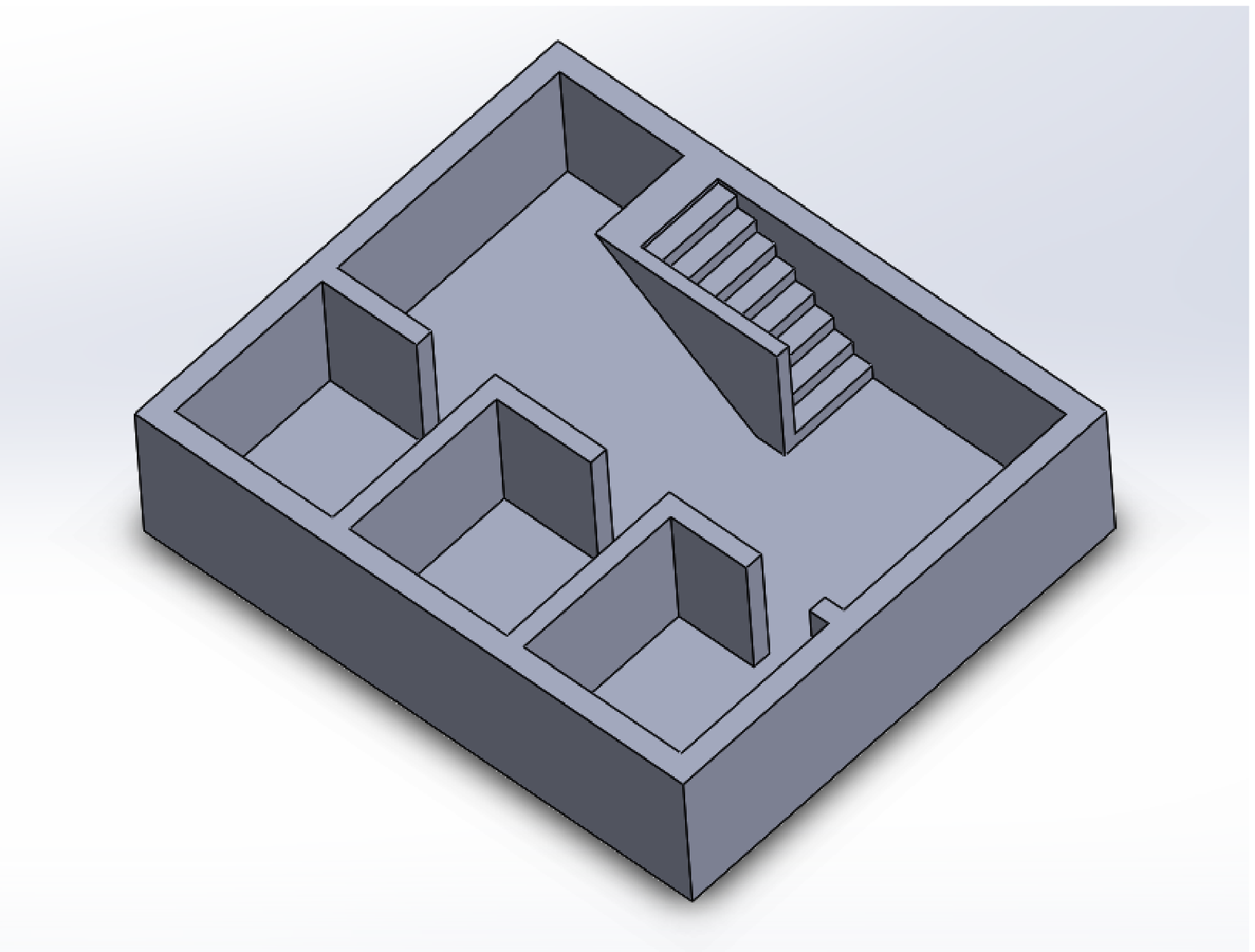}\label{fig:TwoStoryLFEnv}}
\caption{Environments}
\label{fig:Environments}
\end{figure}

\begin{table*}
    \centering
    \begin{tabular}{|c|c|c|c|c|}
         \hline
         Description & Sensing radius & OccMap Resolution & Probability Threshold & Velocity Gain\\
         \hline 
         Parameter & r (m) & $m_r$(cells/m) & $\Bar{a}$ &$K_u$\\
         \hline
         Value & 3 & 4 & 0.5 & 1 \\
         \hline
         \hline
         Description & parameter of s & Bounding Constant of k & parameters of $c,\mu$ & parameter of $\mu$\\
         \hline 
         Parameter  & $R_1(m)$ & $\Bar{k}$ & $\varepsilon_1,\varepsilon_w,\varepsilon_2$ & $\mu_1$  \\
         \hline
         Value & 0.2 & 1 & 0.01 & 10/8 \\
         \hline
    \end{tabular}
    \caption{Controller Parameters}
    \label{tab:ControllerParameters}
\end{table*}

\subsection{Generated Maps and Trajectories}
The generated maps and trajectories are shown in Figures \ref{fig:MapTrajectoryCorridorMaze} and \ref{fig:TwoStoryMapTrajectory} and the corresponding path lengths and exploration times in Table \ref{tab:PathLengthExpTime}. We see that the robot has fully explored the testing environments in reasonable time and maintains a safe distance from the explored space's boundaries. Regarding the Corridor Maze environment, we see that the dead end closing algorithm helped reduce the total exploration time by a factor of 1.6. As for the two story building environment the robot was able to revisit the upper and lower floor to explore remaining space. This validates the completeness of the proposed exploration scheme for large and complex environments. Regarding the performance of FastBEM, the computation of the unknown boundary values in \eqref{PotentialOnBoundary} was conducted every 10 iterations or when the robot approached close to the free boundary and rose linearly from 2 sec for 2500 elements to 27 sec for 27000 elements for the two story building. For the corridor maze those numbers where lower than 6 sec and lower than 2 sec when we applied the dead end closing algorithm. The maximum computation time of the gradient $\nabla_p\varphi$ from \eqref{PotetialInsideDomain} per step was 0.2 and 0.12 sec for the Two-story and the Corridor Maze, respectively. For the latter, the dead end closing algorithm reduced the gradient computation time below 0.04 seconds.

\begin{figure}
    \centering
    \subfigure[Without Dead-end Closing]{\includegraphics[width=0.45\linewidth]{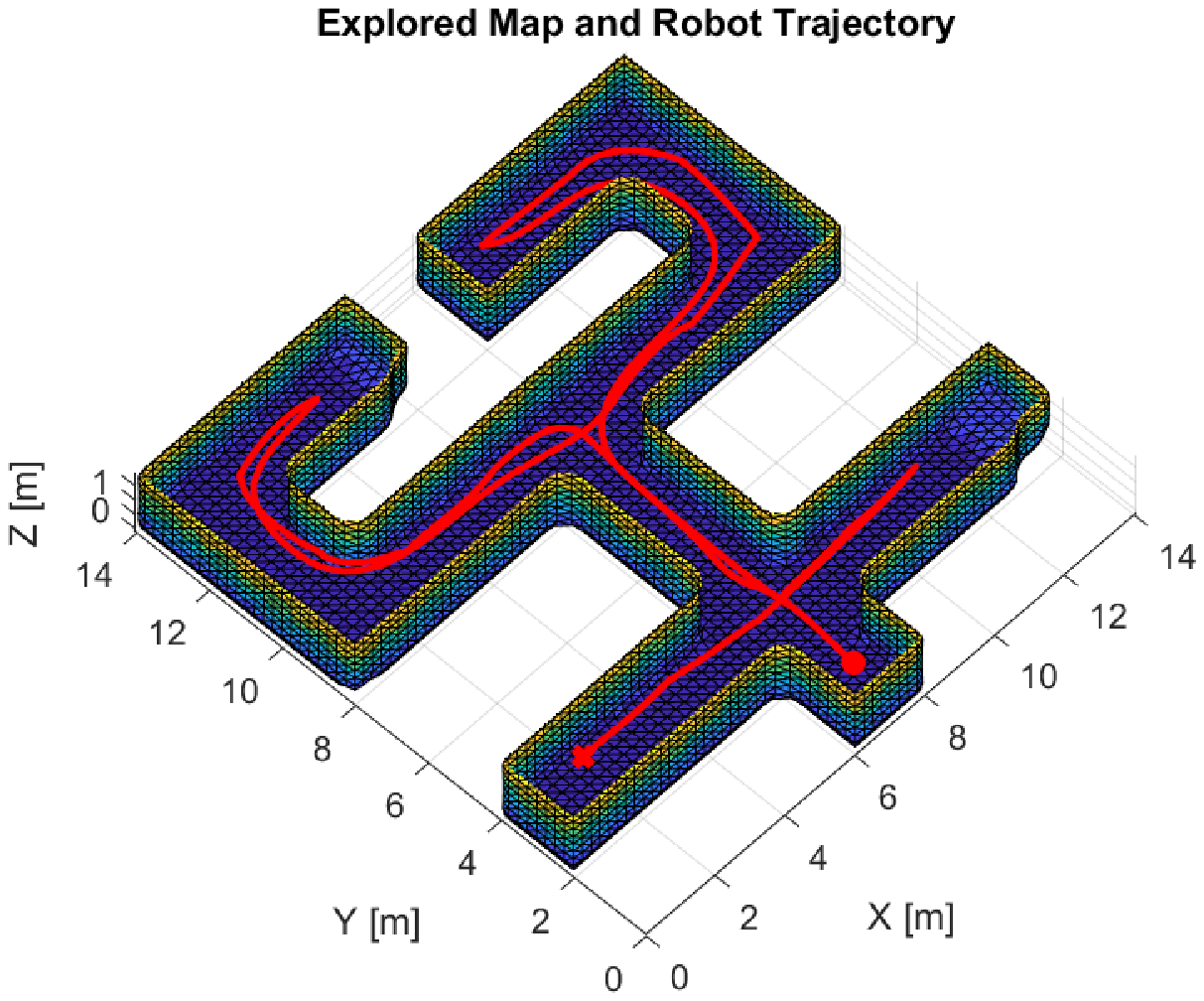}
    \label{fig:CorridorMazeWithoutDEC}}
    \subfigure[With Dead-end Closing]{\includegraphics[width=0.45\linewidth]{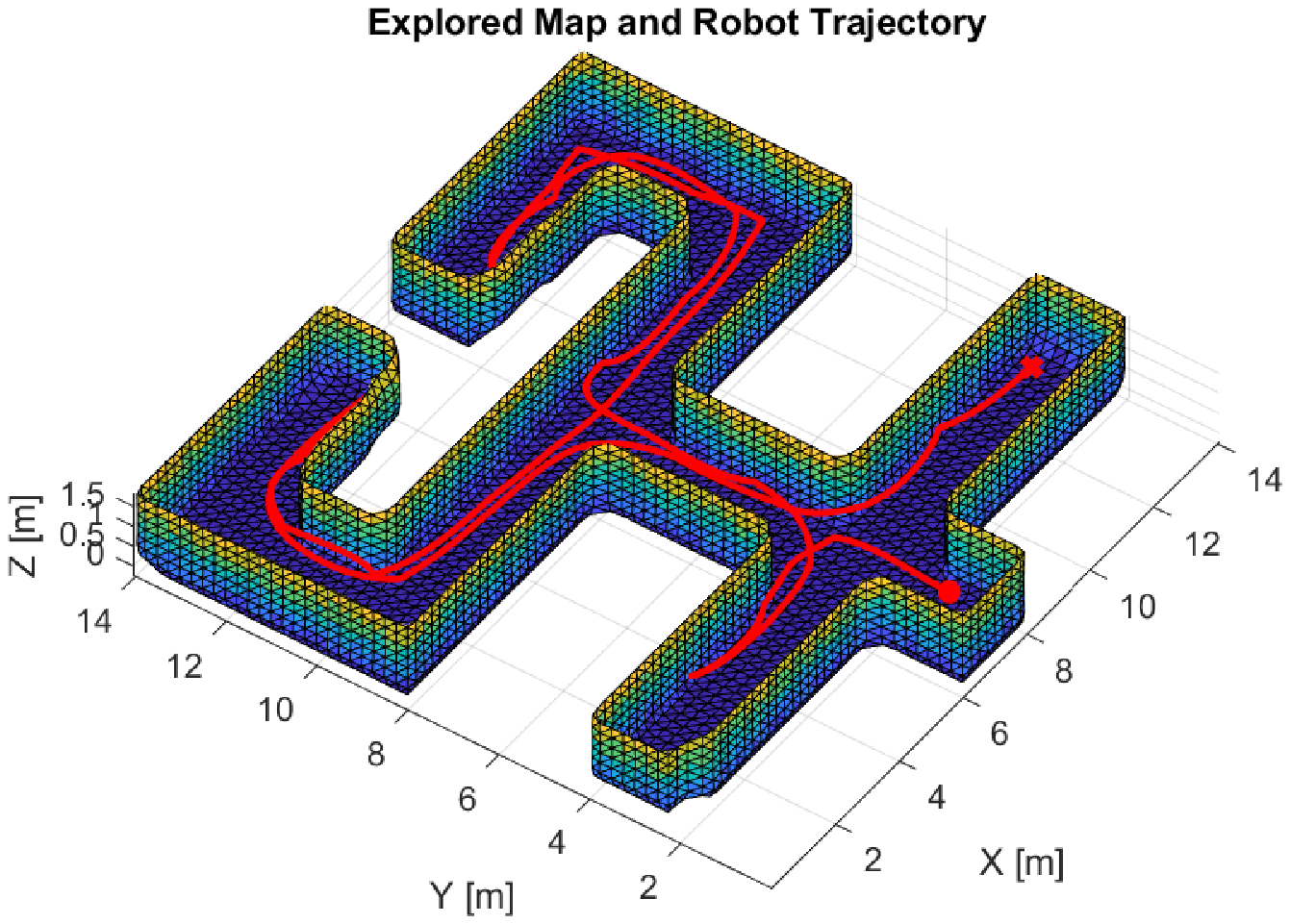}
    \label{fig:CorridorMazeWithDEC}}
    \caption{Corridor Maze: Maps and Trajectories}
    \label{fig:MapTrajectoryCorridorMaze}
\end{figure}

\begin{figure}
    \centering
\subfigure[Upper Floor]{\includegraphics[width=0.45\linewidth]{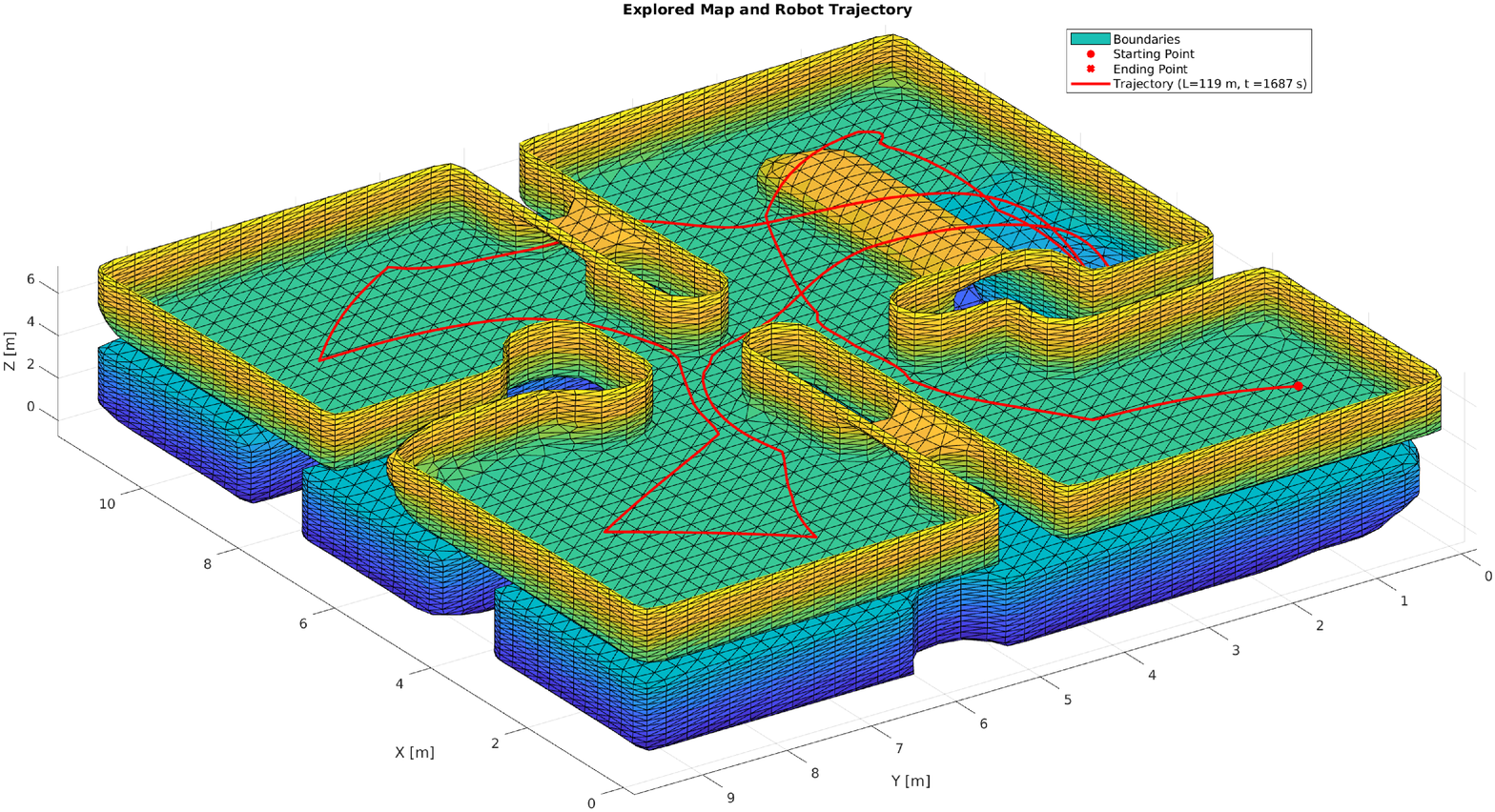}\label{fig:TwoStoryUpper}}
\subfigure[Lower Floor]{ \includegraphics[width=0.45\linewidth]{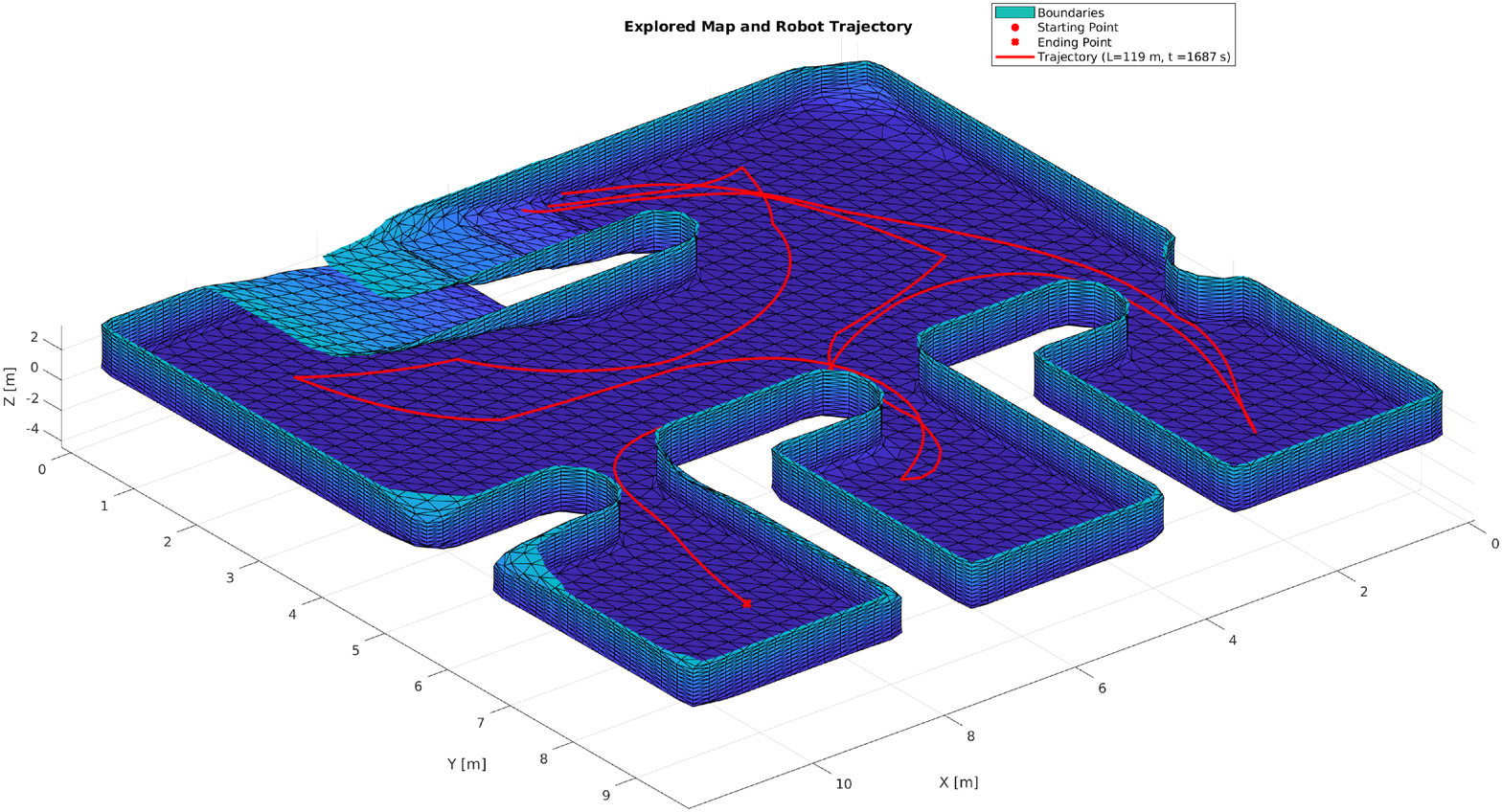}\label{fig:TwoStoryLower}}
    \caption{Two Story Building: Map and Trajectory}
    \label{fig:TwoStoryMapTrajectory}
\end{figure}

\begin{table}
    \centering
    \begin{tabular}{|c|c|c|c|}
         \hline
         Environment & Path Length (m) & Exploration Time (s)\\
         \hline Two-Story Building & 119 & 1687\\
         \hline
         Corridor Maze & 74 & 355\\
         \hline
         Corridor Maze * & 76 & 226 \\
         \hline
    \end{tabular}
    \begin{tabular}{c}
     * denotes the use of the dead-end closing algorithm \\   
    \end{tabular}
    \caption{Path Length and Exploration Time}
    \label{tab:PathLengthExpTime}
\end{table}

\subsection{Comparison}
In this subsection, we compare the performance of our algorithm with two other state-of-the-art algorithms: the receding horizon next best view planner (RH-NBVP) \cite{Bircher} and the frontier-based planner \cite{Gonzalez}, on the same indoor environment. In order to compensate for the narrower field of view used in \cite{Bircher}, which was $[60^\circ,90^\circ]$ with a depth of 5 meters, we used a smaller sensing radius of 3 meters instead of 5 meters. Additionally, we lowered our map's resolution to 3 cells/m to ensure compatibility with the RH-NBVP's map resolution of 2.5 cells/m. The generated maps and trajectories for our algorithm are shown in Figure \ref{fig:ApartmentOurs} for two different starting positions, while those for the RH-NBVP and the frontier-based planner are shown in Figure \ref{fig:ApartmentNBVPandfrontier}. Regarding the exploration time, the overall time of the RH-NBVP was 501 seconds, the frontier-based planner was 470 seconds (on average), and our algorithm was faster than the other two algorithms, with 451 seconds and 371 seconds respectively for the two initial points. All approaches fully explored the environment, but the trajectories generated by our algorithm are smoother (smaller radius of curvature more than 0.45m), maintain a safe distance from the walls more than 0.35m and transverse every corridor of the environment. In contrast, the trajectories generated by the other approaches contain sharp turns, approach close to the walls (the minimum distance reached 0.05m) and do not extend evenly to all the corridors of the environment.

\begin{figure}
    \centering
    \includegraphics[width=0.95\linewidth]{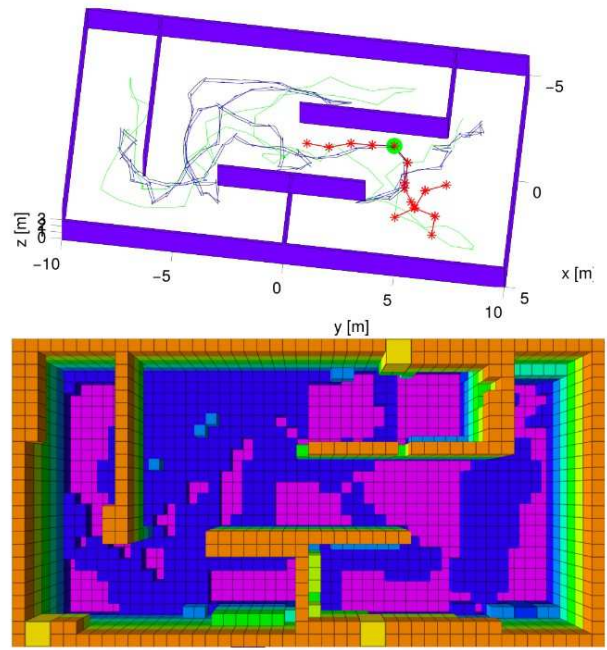}
    \caption{Apartment: Map and Trajectories generated by RH-NBVP (blue path) and the frontier based algorithm (green path). Figure extracted by \cite{Bircher}.}
    \label{fig:ApartmentNBVPandfrontier}
\end{figure}

\begin{figure}
    \centering
    \subfigure{\includegraphics[width=0.45\linewidth]{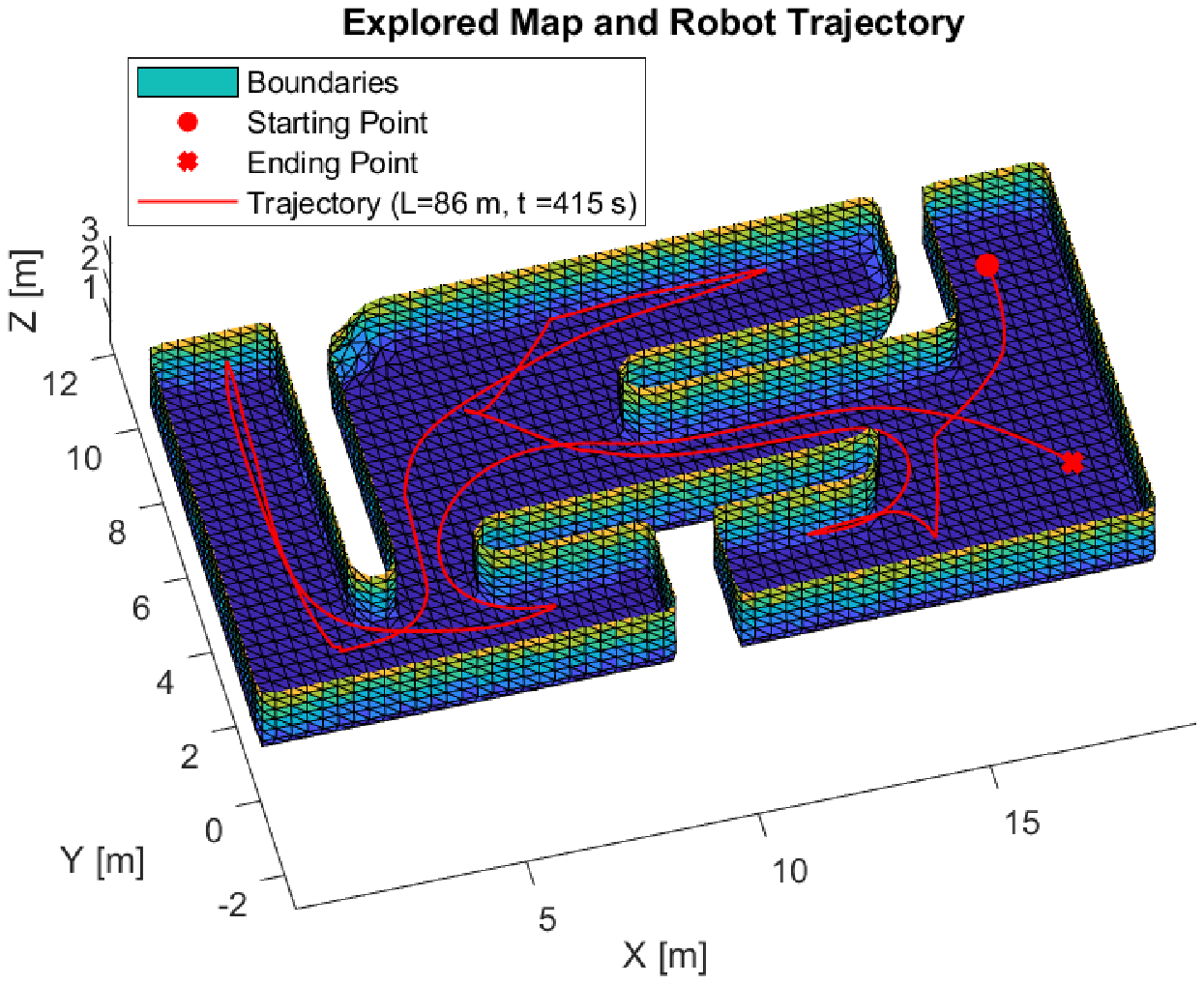}}
    \subfigure{\includegraphics[width=0.45\linewidth]{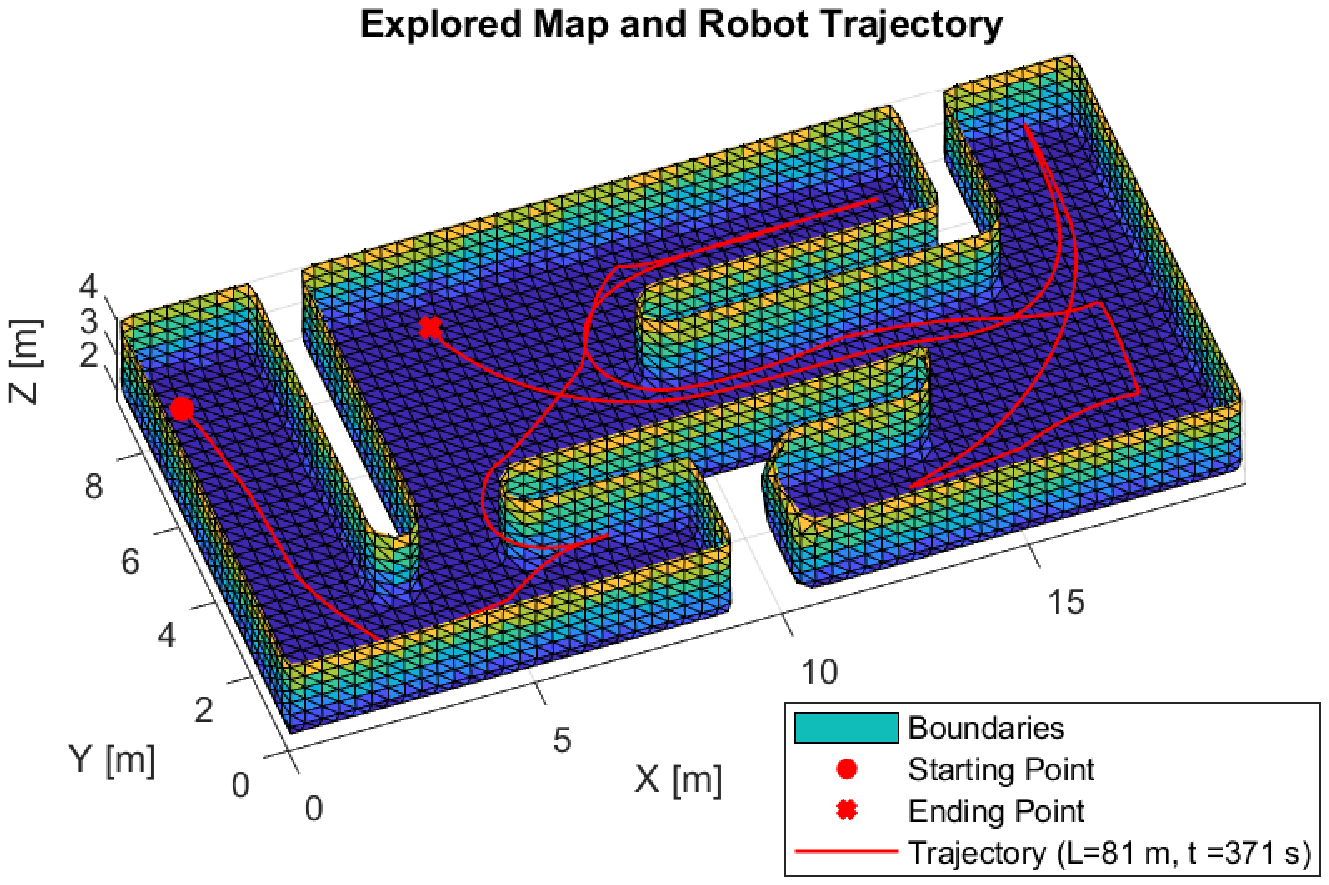}}
    \caption{Apartment: Maps and Trajectories generated by our algorithm for two different initial positions}
    \label{fig:ApartmentOurs}
\end{figure}

\section{Conclusions}
In summary, this work presents a safe, smooth, efficient and provably complete method for 3D exploration of complex indoor environments. Furthermore, we proposed two algorithms for robust boundary extraction and for reduction of the boundary's size by removing explored dead ends automatically. The aforementioned properties were validated by simulation results in various types of environments. In comparison to two existing approaches, we showed that the proposed scheme outperforms them in terms of safety, trajectory smoothness and exploration time in a typical indoor environment. For future research, we aim at applying a faster GPU-based version of FMBEM in real experiments and extending the methodology to outdoor environments for tasks like inspection of structures.

\bibliography{References}
\bibliographystyle{ieeetran}

\end{document}